\documentclass{article}

\usepackage[final]{neurips_2019}

\usepackage{microtype}
\usepackage{graphicx}
\usepackage{wrapfig}
\usepackage{subcaption}
\usepackage{booktabs} % for professional tables

\usepackage{algorithm}
\usepackage[noend]{algorithmic}

\usepackage{amsmath}
\usepackage{amssymb}
\usepackage{mathrsfs}
\usepackage{amsthm}
\usepackage{stmaryrd}
\usepackage{mathtools}

\usepackage{bm}
\usepackage{dblfloatfix}

\usepackage{thmtools, thm-restate}

\declaretheorem{proposition}

\usepackage{tikz}
\usetikzlibrary{automata, positioning, arrows}
\tikzset{
  ->,
  >=stealth',
  node distance=3cm,
  every state/.style={thick, fill=gray!10},
  initial text=$ $,
}

\usepackage{listings}
\usepackage{color}
\definecolor{codegreen}{rgb}{0,0.6,0}
\definecolor{codegray}{rgb}{0.5,0.5,0.5}
\definecolor{codepurple}{rgb}{0.58,0,0.82}
\definecolor{backcolour}{rgb}{0.95,0.95,0.92}

\lstdefinestyle{mystyle}{
    aboveskip=1em,
    belowskip=0.2em,
    backgroundcolor=\color{backcolour},   
    commentstyle=\color{codegreen},
    keywordstyle=\color{magenta},
    numberstyle=\tiny\color{codegray},
    stringstyle=\color{codepurple},
    basicstyle=\footnotesize\ttfamily,
    breakatwhitespace=false,         
    breaklines=true,                 
    captionpos=b,                    
    keepspaces=true,                 
    numbers=left,                    
    numbersep=3pt,                  
    showspaces=false,                
    showstringspaces=false,
    showtabs=false,                  
    tabsize=2
}
 
\usepackage{hyperref}
\usepackage{cleveref}

\newcommand{\cA}{\mathcal{A}}

\newcommand{\cP}{\mathcal{P}}
\newcommand{\cR}{\mathcal{R}}
\newcommand{\cS}{\mathcal{S}}

\newcommand{\rE}{\mathbb{E}}

\newcommand{\rR}{\mathbb{R}}

\title{Inverse {R}einforcement {L}earning with {M}ultiple {R}anked {E}xperts}

\author{
  Pablo Samuel Castro\thanks{Work done while at Institut Mines-TELECOM/TELECOM SudParis} \\
  Google Brain \\
  \texttt{psc@google.com} \\
  \And
  Shijian Li \\
  Zhejiang University
  \And
  Daqing Zhang \\
  Institut Mines-TELECOM\\TELECOM SudParis
}

\begin{document}

\maketitle

\begin{abstract}
  We consider the problem of learning to behave optimally in a Markov Decision
  Process when a reward function is not specified, but instead we have access
  to a set of demonstrators of varying performance. We assume the demonstrators
  are classified into one of $k$ ranks, and use ideas from ordinal regression
  to find a reward function that maximizes the {\em margin} between the
  different ranks.  This approach is based on the idea that agents should not
  only learn how to behave from experts, but also how {\em not} to behave from
  non-experts. We show there are MDPs where important differences in the reward
  function would be hidden from existing algorithms by the behaviour of the
  expert. Our method is particularly useful for problems where we have access
  to a large set of agent behaviours with varying degrees of expertise (such as
  through GPS or cellphones). We highlight the differences between our approach
  and existing methods using a simple grid domain and demonstrate its efficacy
  on determining passenger-finding strategies for taxi drivers, using a large
  dataset of GPS trajectories.
\end{abstract}

\section{Preamble/Disclaimer}
{\em This preamble was written by Pablo Samuel Castro}

This is a paper I wrote during my postdoc with Daqing Zhang at the Institut
Mines-TELECON/TELECOM SudParis.  We submitted it to NIPS (now called
NeurIPS) 2011 and it was rejected. I finished my postdoc shortly after the rejection
and joined Google as a software engineer, saying goodbye to academia for ever
(at the time I couldn't envision a future where I'd be lucky enough to re-join the
academic world). This was at a time when putting papers up on ArXiv was not a
common thing (at least not for me), so the paper lay dormant in my e-mail.

During a recent visit to Scott Niekum's group in UT Austin I decided to dust it
off, as they're doing research on inverse reinforcement learning. One of his
students, Daniel Brown\footnote{Check out some of his recent papers, including \citep{brown19extrapolating,brown18efficient}.}
mentioned that he liked the proposition I have in this paper, and wondered how
he could cite it. And that is what led me to put this up.

This paper is presented as-is: except for a few minor grammatical corrections,
it is unchanged from my original submission. I considered revising it and
bringing the related-work section up-to-date, but upon re-reading it there are
many things I feel would need to be changed, updated, or completely rewritten.
In short: this is not a paper I would send out for review in its current form.
\footnote{Some of the things I would do are: Include an appendix with all the
results discussed but not included; fix the colors on the GridWorld plots (they
make no sense as they are right now!); run on more domains.}

So why am I putting up on ArXiv then? I do feel that some of the ideas
presented here are interesting enough that they could perhaps be of value to
other researchers in the field (at least \autoref{eqn:prop} seemed interesting
enough to Daniel to be worth citing).

However, I am not currently working on this problem, I no longer have access to
the code I wrote for it, nor to the GPS data I used, so it is extermely
unlikely that I will work on an updated-and-improved version of this paper. If
you find this idea interesting and want to make a proper paper out of it, be my
guest!

\section{Introduction}
\label{sec:intro}
Markov Decision Processes (MDPs) are a popular mathematical framework for
encoding a sequential decision making problem, where an agent aims to find an
optimal way of acting in said environment. The optimality of the agent’s
behaviour has traditionally been quantified by means of a reward/cost function,
supplied as part of the MDP. Whether the optimal behaviour is computed directly
or learned through experience, the reward function plays a central role in this
process, enabling the agent to rank the actions according to their expected
numerical return.

Many problems have been successfully encoded as MDPs and have produced
successful behaviours.  Nevertheless, there are many problems for which
specifying a reward function can be very difficult.  Consider the problem of
learning to ride a bicycle: although a reward function could be specified that
gives a penalty for falling down or losing one’s balance, there are many other
factors which characterize a good rider. Indeed, when we are taught to ride a
bicycle, our teachers describe or demonstrate the type of behaviour that will
enable us to ride the bicycle properly. In these situations, it is often easier
to demonstrate an optimal behaviour, and have the agent attempt to emulate it.
This problem is known as imitation learning, apprenticeship learning, learning
from demonstration, amongst others, and has been studied extensively during the
last two decades, often for robotics applications using supervised learning
methods \citep{hayes94robot,amit02learning,pomerleau89alvinn}. More recently, \cite{abbeel04apprenticeship} proposed an apprenticeship learning
algorithm that guarantees the learned policy will be close to the expert’s
policy.

A similar approach is to construct a reward function under which the expert’s
behaviour is optimal.  In the Artificial Intelligence community this is known
as the Inverse Reinforcement Learning (IRL) problem and was first formally
studied in \citep{ng00algorithms}. This approach has the advantage of producing a reward function
which can be used afterwards (such as for Reinforcement Learning (RL)
algorithms), and reward functions are usually a more compact representation of
a domain than a policy.

Both approaches mentioned above are based on the observed behaviours of an
expert.\footnote{In most cases a single expert is considered, although \citep{ratliff06maximum} is
a notable exception.} In other words, the sole purpose of the demonstrations is
to inform the agent of what it should do. This is somewhat in contrast to the
way humans learn: we learn to behave not only by attempting to imitate experts,
but also by avoiding the behaviours of non-experts. The field of imitation
learning is based on the assumption that the expert’s behaviour is optimal with
respect to an underlying reward function, and as long as there are enough
expert demonstrations to “cover” the environment, the extracted reward function
should be sufficient to produce a similar behaviour. This may be so, but as we
will prove below, the resulting reward function may be rather superficial, as
it may fail to disclose certain “negative” aspects of the true reward function.
Moreover, if the expert is aware of certain areas in the environment that have
low reward, she may avoid them altogether, resulting in incomplete information
about the environment.

It is usually assumed that an expert is ``requested'' to act in an enviornment,
from which the behaviours can be extracted. This is not problematic in
simulated or controlled environments, but may be difficult in large, real-world
problems. Nevertheless, due to the growing pervasiveness of devices capable of
logging a user’s contextual information (such as location, actions, etc.), we
may have access to many experts and non-experts. This paper introduces a method
for Inverse Reinforcement Learning using multiple experts with varying degrees
of expertise. Our approach makes use of maximum margin ordinal regression
algorithms \citep{shashua02ranking} to produce a reward function that preserves the ranking of our
experts, as well as maximizing the margin between the different ranks.  We
illustrate our approach on the problem of finding optimal taxi driver
strategies for finding new passengers, where we make use of the driving
trajectories from multiple GPS-equipped taxis. As we will demonstrate, our
approach is able to produce meaningful and useful results.

The paper is organized as follows. In \autoref{sec:preliminaries} we introduce
Markov Decision Processes and discuss the problem of apprenticeship learning,
along with related works. We present our proposed method in
\autoref{sec:rankIRL}, along with a proof of the “superficiality” of solutions
based solely on expert’s behaviours. In \autoref{sec:illustration} we highlight
the advantages of our approach over existing approaches using a simple grid,
and demonstrate the efficacy of our method on a large problem involving
providing passenger-finding strategies to taxi drivers in \autoref{sec:taxi}.
We conclude our results and discuss future avenues of research in
\autoref{sec:conclusion}.

\section{Preliminaries}
\label{sec:preliminaries}
\subsection{Markov Decision Processes}
A finite Markov Decision Process (MDP) is a 4-tuple $\langle\cS, \cA, \cP,
\cR\rangle$, where $\cS$ is a finite set of states, $\cA$ is a finite set of
actions available at each state, $\cP:\cS\times\cA\rightarrow
Dist(\cS)$\footnote{$Dist(X)$ is the set of all probability distributions on a
set $X$.} is a probabilistic transition function, and $\cR:\cS\rightarrow\rR$
is the reward function.

The behaviour of an agent in an MDP is formalized as a policy
$\pi:\cS\rightarrow\cA$; let $\Pi$ be the set of all policies. We define the
value of a policy $\pi$ from a state $s_0$ as $V^{\pi}_{s_0} =
\rE\left[\sum_{t=0}^{\infty}\gamma^t r_t | s_0, \pi\right]$, where
$0\leq\gamma < 1$ is a discount factor and $r_t$ is a random variable
representing the reward received at time step $t$, given that we started at
state $s_0$ and followed policy $\pi$. We may sometimes refer to $V^{\pi}$ as
the $|S|$-dimensional vector containing the values for all states.

In most situations, one is interested in finding an {\em optimal} behaviour:
a policy $\pi^*$ satisfying $V^{\pi^*}\geq V^{\pi}$ for all $\pi\in\Pi$.
If the MDP parameters are known, there are a number of exact and approximate
methods for finding $\pi^*$ \citep{puterman94mdp}. Reinforcement learning is a popular approach
to learning to act optimally when neither the transition or reward functions are
known {\em a priori}, but are revealed to the agent as it interacts with the
environment \citep{sutton98rl}.

\subsection{Apprenticeship Learning}
In the situations mentioned above, it is implicitly assumed that the agent has
access to a fully specified reward function. Although this is a reasonable
assumption in many problems, there are many situations where specifying a reward
function can be difficult. In these situations, it may be easier to specify or
demonstrate a desired behaviour, rather than specify a reward function.

For the rest of the paper we will assume that MDPs are not equipped with a
reward function; when a reward function $\cR$ is specified, we will denote the
resulting value function under policy $\pi$ as $V^{\pi}_{\cR}$. The Inverse
Reinforcement Learning (IRL) problem aims to determine a reward function that
``explains'' the behaviour, or policy $\pi_E$, of an expert $E$. Specifically,
it aims to find a reward function $\cR$ such that $\pi_E =
\arg\max_{\pi\in\Pi}V^{\pi}_{\cR}$. In \citep{ng00algorithms} the authors characterize the set of
reward functions with this property. They point out that many of these reward
functions will be degenerate, and suggest some heuristics to overcome this
degeneracy. The heuristics are based on the principle of choosing a reward
function that maximizes the difference between the value of the expert's policy
and other, sub-optimal, policies. In \citep{ziebart08maximum} the authors apply the principle of
maximum entropy to overcome this degeneracy.

This maximum margin principle lies at the heart of the algorithms proposed in
\citep{abbeel04apprenticeship} and \citep{ratliff06maximum}. \cite{abbeel04apprenticeship} assume there is a feature vectore
$\phi:\cS\rightarrow [0,1]^d$ and that the reward function can be expressed as
a linear combination of these features: $\cR(s) = w\cdot\phi(s)$, where
$w\in\rR^d$ is a weight vector. Under this assumption, it can be shown that
\[ \rE_{s_0\sim D}[V^{\pi}(s_0)] = w\cdot\rE\left[\sum_{t=0}^{\infty}\gamma^t \phi(s_t) | \pi\right] = w\cdot\mu(\pi) \]
where $D$ is an initial state distribution, $s_t$ is a random variable
representing the state at time step $t$ when starting at $s_0\sim D$ and
following policy $\pi$, and $\mu(\pi)$ are the {\bf feature expectations} of
policy $\pi$. The authors assume access to an estimate of the expert's feature
expectations, $\mu_E = \mu(\pi_E)$. Their algorithm begins with a random policy
$\pi_0$, and at each iteration $i$ chooses a weight vector $w_i$ that maximizes
the margin $t_i$ between $w\cdot\mu_E$ and $w\cdot\mu(\pi_j)$ for all $j < i$;
$\pi_i$ is then set to be an optimal policy for the MDP augmented with the
reward function $\cR = w_i\cdot\phi$ and the process is repeated until $t_i$
goes below a pre-specified threshold. As the authors point out, at each
iteration we are finding the maximum margin hyperplane separating $\mu_E$ from
the set of $\mu_i$, for which a Support Vector Machine (SVM) or Quadratic
Programming (QP) solver can be used.

The maximum margin approach was further studied in \citep{ratliff06maximum}, where the authors use
a pre-defined loss function that quantifies the ``closeness'' of the computed
policy to the expert's policy. The chosen loss function directly impacts the
performance of the algorithm; indeed, the resulting margin will scale with the
loss function. Rather than starting from a single expert and computing new
policies at each iteration as in \citep{abbeel04apprenticeship}, the authors assume access to a set of
experts (each potentially defined in a separate MDP). Their algorithm seeks to
find a weight vector $w$ such that the optimal policy resulting from each input
MDP augmented with the reward function $w\cdot\phi$ is ``close'' to the supplied
expert's behaviour.

\section{Rank IRL}
\label{sec:rankIRL}
As mentioned in the introduction, there are many situations where we cannot
simply request an expert to interact with an environment in order to observe
their behaviour. However, with the surge in ubiquitous devices such as GPS and
mobile phones, we may have access to many experts. Although we may combine to
obtain an ``average'' expert, we can capitalize on the difference amongst the
various experts, and in particular, non-experts. The following result
demonstrates that the reward functions produced by an IRL or apprenticeship
learning solver may be quite superficial.

\begin{proposition}
  \label{eqn:prop}
  There exists an MDP with true reward function $\cR^*$, expert policy $\pi_E$,
  approximate reward function $\hat{\cR}$, and non-expert policies $\pi_1$ and
  $\pi_2$ such that
  \begin{align*}[lr]
    \pi_E = \arg\max_{\pi\in\Pi} V^{\pi}_{R^*} \qquad\qquad\qquad\qquad & V^{\pi_1}_{\cR^*} \ll V^{\pi_2}_{\cR^*} \\
    \pi_E = \arg\max_{\pi\in\Pi} V^{\pi}_{\hat{\cR}} \qquad\qquad\qquad\qquad & V^{\pi}_{\hat{\cR}} = V^{\pi_2}_{\hat{\cR}}
  \end{align*}
\end{proposition}
\begin{proof}
  Consider the following MDP with deterministic transitions. Each transition is
  labelled by the action name, and the true reward (if any) received upon
  entering a state is indicated in the state's node.

  \begin{figure}[!h]
    \centering
    \begin{tikzpicture}
      \node[state] (s0) {$s_0$};
      \node[state, right of=s0] (s2) {$s_2 (+1)$};
      \node[state, above of=s2] (s1) {$s_1 (-\delta)$};
      \node[state, below of=s2] (s3) {$s_3 (0)$};

      \draw (s0) edge[above] node{a} (s1)
            (s0) edge[above] node{b} (s2)
            (s0) edge[above] node{c} (s3)
            (s1) edge[loop right] node{a, b, c} (s1)
            (s2) edge[loop right] node{a, b, c} (s2)
            (s3) edge[loop right] node{a, b, c} (s3);
    \end{tikzpicture}
  \end{figure}
  Here $\delta>0$ is some arbitrary constant. Obviously $\pi_E(s_0) = a$, and
  it is easy to see that $\hat{R}(s_1) = 1$, $\hat{R}(s_2) = \hat{R}(s_3) = 0$,
  $\pi_1(s_0) = b$ and $\pi_2(s_0) = c$ yield the desired result.
\end{proof}

Our method will make use of feature expectations, as policies are usually
difficult to specify precisely. Let us assume that we have access to the
expected features of a number of demonstrators: $\lbrace
\mu_i\rbrace_{i=1}^{N}$, and to a ranking function $rand(\mu_i) = 1, \cdots, k$
that classifies each expected feature into one of $k$ ranks. Let us also define
the set $rank_r = \lbrace\mu_i | rank(\mu_i) = r\rbrace$.  We wish to find a
weight vector $w^*$ such that for all $i\ne j, w\cdot\mu_i < w\cdot\mu_j$ if
and only if $rank(\mu_i) < rank(\mu_j)$. For all $1\leq r < k$ we would also
like to maximize $\rho_r$, where
\[ \rho_r = \min_{\mu_i\in rank_r,\mu_j\in rank_{r+1}}w\cdot\mu_j - w\cdot\mu_i \]
In other words, $w$ should maximize the difference between ranks.

Ranking the instances properly is an instance of the ordinal regression
problem, and applying the maximum margin principle yields the following
(primal) quadratic program (QP):
\begin{align}
  \label{eqn:qp}
  \min_{w,a_r,b_r,\varepsilon^r_i,\varsigma^{r+1}_i} & \sum_{r=1}^{k-1}(a_r - b_r) + C\sum_{r}\sum_{i}(\varepsilon_i^r + \varsigma_{i}^{r+1}) & \\
  \textrm{subject to} \qquad & a_r\leq b_r & \forall\quad 1\leq r < k \nonumber \\
  & b_r \leq a_{r+1} & \forall\quad 1\leq r < k - 1 \nonumber \\
  & w\cdot\mu_i\leq a_r + \varepsilon^r_i & \forall\quad\mu_i\in rank_r \nonumber \\
  & b_r - \varsigma^{r+1}_i\leq w\cdot\mu_i & \forall\quad\mu_i\in rank_{r+1} \nonumber \\
  & \| w\|\leq 1 & \nonumber \\
  & \varepsilon^r_i,\varsigma^{r+1}_i\geq 0 & \forall\quad 1\leq r < k \nonumber
\end{align}

\begin{figure}
  \centering
  \includegraphics[width=0.3\textwidth]{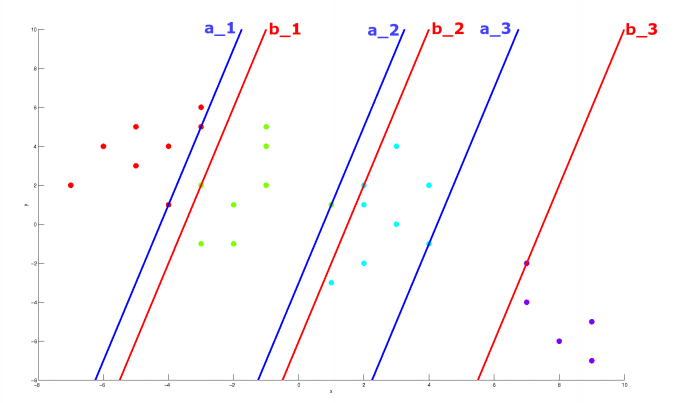}
  \caption{Three pairs of hyperplanes separating 4 ranks of points. TODO: make
  a prettier version of this.}
  \label{fig:hyperplanes}
\end{figure}

This is the sum-of-margins strategy introduced in \citep{shashua02ranking}. Traditional SVM solvers
maximize a single margin between two classes of points (a notion which is
generalized to ordinal regression in \citep{shashua02ranking} as the fixed-margin strategy). In our
current situation we wish to maximally distinguish between ranks, which
corresponds to maximizing $k-1$ margins. Rather than searching for a single
separating hyperplane, we are searching for $2(k - 1)$ parallel hyperplanes
with normal vector $w$ (see \autoref{fig:hyperplanes}). Note that one of the
constraints in QP \autoref{eqn:qp} is that the norm of $w$ is at most one. In
\citep{shashua02ranking} they demonstrate that using this convex constraint is equivalent to using
the non-convex constraint $\|w\| = 1$.  Hence, the margin between ranks $r$ and
$r+1$ is given by $b_r - a_r$. The performance of our method depends directly
on the quality of the supplied ranking. The slack variables $\varepsilon^r_i$
and $\varsigma^{r+1}_i$ allow for some instances to be misclassified or be on
the wrong side of $a_r$ or $b_r$, respectively. The constant $C$ trades off
between the sum of the margins and margin errors (see \citep{shashua02ranking} for a lengthier
discussion of the meaning of $C$). We set $C = 1$ for all of our experiments
below. In the conclusion we suggest some simple mechanisms for dealing with
unreliable rankings.

\begin{figure}
  \begin{subfigure}{0.5\textwidth}
    \centering
    \includegraphics[width=\textwidth]{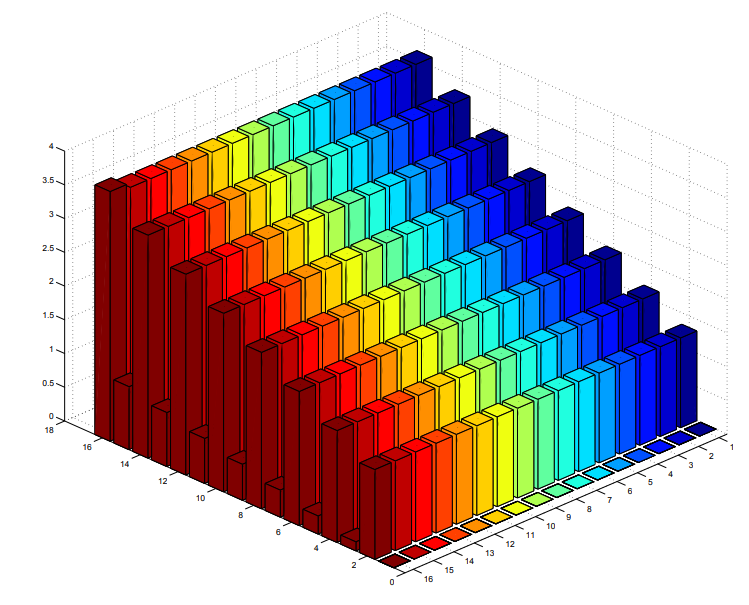}
  \end{subfigure}
  \begin{subfigure}{0.5\textwidth}
    \centering
    \includegraphics[width=\textwidth]{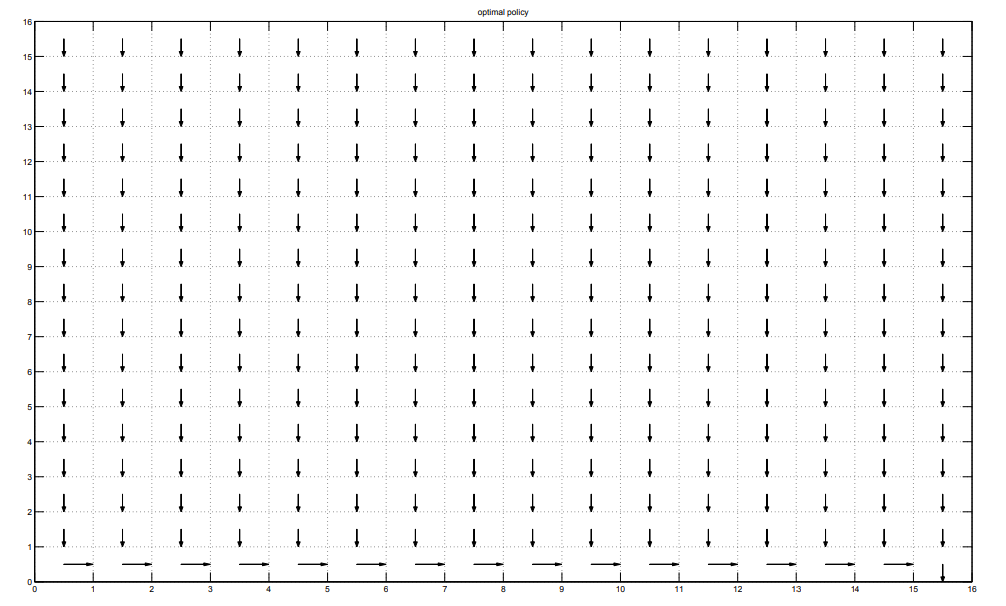}
  \end{subfigure}
  \caption{Reward function (left) and optimal policy (right) for 16x16 grid.
  The reward function plot has been rotated for clarity; the top left state in
  the left panel corresponds to the bottom-right state in the right panel.
  TODO: make these prettier.}
  \label{fig:gridworld}
\end{figure}

\section{Illustration}
\label{sec:illustration}
As mentioned above, our method is motivated by the fact that agents can learn
how not to act from sub-optimal demonstrations. \autoref{eqn:prop}
demonstrates that the policy/behaviour of an expert can obfuscate certain
aspects of the true underlying reward function. We illustrate this point
further by comparing our method against the apprenticeship learning algorithm
of \cite{abbeel04apprenticeship} on a simple grid domain.  Consider a 16x16 room with reward function as
given in the left panel of \autoref{fig:gridworld} and optimal policy in the
right panel. As can be seen, even though there is a clear advantage to being in
the even rows as opposed to the odd rows, this advantage is not evident in the
optimal policy. This optimal policy was used to generate the expert’s expected
features for both algorithms.

We ran the algorithm from \cite{abbeel04apprenticeship} against our method (henceforth referred to as
RankIRL). We used a ``lossless'' feature vector, that is, one feature for every
grid cell. Thus, the $w$ returned by the algorithms would exactly correspond to
the reward function. We used four types of policies to generate three types of
expected features for our method:
\begin{itemize}
  \item \textbf{First rank (expert):} The optimal policy 
  \item \textbf{Second rank:} A policy that avoids odd rows and goes right
    only on even rows, then goes down once in the last column
  \item \textbf{Third rank:} The policy is similar to the second, but even
    rows are avoided and odd rows are followed towards the right
  \item \textbf{Fourth rank:} A policy that avoids even rows and follows odd
    rows to the left
\end{itemize}

We assume a uniform distribution over all possible cells for both algorithms.
In \autoref{fig:gridworld_results} we compare the reward functions of the two
algorithms, where it is evident that our method produces a reward function that
prefers even rows over odd rows (most evident in the states on the right-hand
side). The algorithm of \cite{abbeel04apprenticeship} displayed much higher variability than our method
due to the random policy used in the first iteration. As is pointed out by the
authors, although their method is guaranteed to produce expected features close
to those of the expert's, it is unfortunate that the resulting vector $w$ is
unreliable. We also computed the ratio of the performance of the policies from
both algorithms against the performance of the true optimal policy, and our
method consistently had an advantage of at least 20\% over the algorithm of \cite{abbeel04apprenticeship}.
This behaviour was consistent as the number of sampled trajectories used to
generate the expected features were varied. A visual inspection of the
resulting policies also reveals that our method explicitly avoids odd rows,
while the policy produced by the algorithm of \cite{abbeel04apprenticeship} is oblivious to the
difference between even and odd states. We have not included these figures for
lack of space.  Finally, we also compared our algorithm against the MMP
algorithm of \cite{ratliff06maximum} and observed similar results. This is not surprising, since
their method is still based on emulating one or more experts.

\begin{figure}
  \begin{subfigure}{0.5\textwidth}
    \centering
    \includegraphics[width=\textwidth]{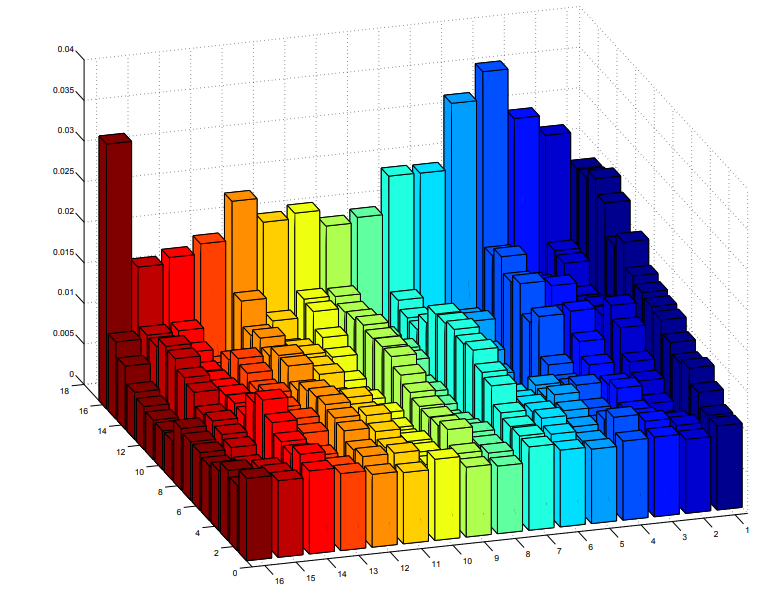}
  \end{subfigure}
  \begin{subfigure}{0.5\textwidth}
    \centering
    \includegraphics[width=\textwidth]{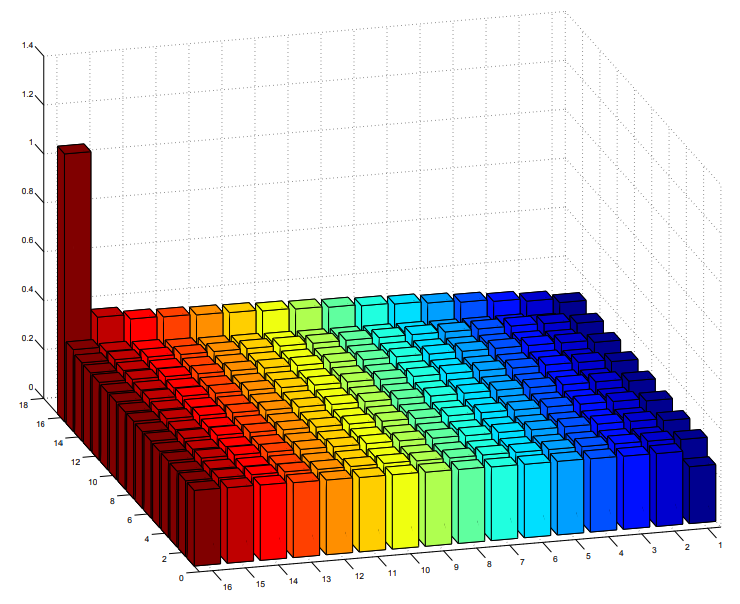}
  \end{subfigure}
  \caption{Reward function returned by the algorithm of \cite{abbeel04apprenticeship} (left) and rankIRL
  (right). These plots have been rotated in the same way as the left panel in
  \autoref{fig:gridworld}.}
  \label{fig:gridworld_results}
\end{figure}

\section{Application: Passenger finding routes for taxi drivers}
\label{sec:taxi}
We have access to a large dataset consisting of 5000 taxi drivers’ GPS data
collected over a year in Hangzhou, China. Amongst other things, each GPS entry
gives us the taxi’s unique ID, its location (in latitude/longitude), speed and
state (occuppied/vacant). This data is logged approximately once per minute. We
constructed a digital representation of the road network of Hangzhou based on
the accumulated taxi trajectories, and mapped each GPS entry to a road segment.
Mapping the analogue trajectories onto a discrete domain essentially removed
noise and small variations in the trajectories.

An important problem for taxi drivers is determining the best strategy to find
a new passenger when unoccupied. Existing approaches include using a naive
Bayesian classifier to predict the number of vacant taxis in different areas
\citep{phithakkitnukoon10taxi}, using hand-filled surveys and simple statistical analyses \citep{takayama11waiting}, using
$k$-means clustering coupled with temporal analysis \citep{hayes94robot}, and using $L1$-norm SVM to
determine whether a taxi driver should ``hunt'' or wait \citep{li11hunting}. The granularity of
these approaches is quite coarse, as most rely on splitting the city into
equal-sized grids. It would be more useful to provide passenger finding
assistance to taxi drivers on a road-by-road basis.

To test our method on this problem, we decided to focus on the weekdays of a
single month. Each taxi driver was ranked according the proportion of time
during the day they spent unoccupied (lower is better), since taxis with a low
proportion probably have a good passenger-finding strategy. The size of the
state space for our MDP was defined to be twice the number of road segments,
resulting in two states for each road segment, one for each
orientation.\footnote{The fact that certain roads may be unidirectional is not
a problem, as we should not have any trajectories going in the wrong
direction!} The road segment and orientation is quite important for this
problem, as traffic conditions and speed limits can vary greatly amongst road
segments, and even within the same segment, depending on orientation. Thus, we
decided to once again use a lossless feature vector, with one feature for every
possible state.

Our road network has 2203 road segments, which means we are searching for a
weight vector $w\in\rR^{4406}$.  Solving a QP in this high dimensional setting
was not possible in our machines\footnote{This was in 2011}, so we took
advantage of the ``Cartesian'' nature of our domain and split the city into
disjoint areas. The areas were obtained dynamically by splitting the city
according to most frequently traversed road intersections.  The details of this
decomposition method will be ommitted as it is outside the scope of this paper.
The decomposition we used is displayed in \autoref{fig:cityDecomposition}, and
all of the trajectories were split according this city decomposition. We
computed a weight vector $w$ for each disjoint cluster and combined them
afterwards. This reduced the dimensionality of our feature vectors from 2203 to
no more that 130 dimensions and enabled us to compute the solutions in
parallel.

\begin{figure}
  \centering
  \includegraphics[width=0.4\textwidth]{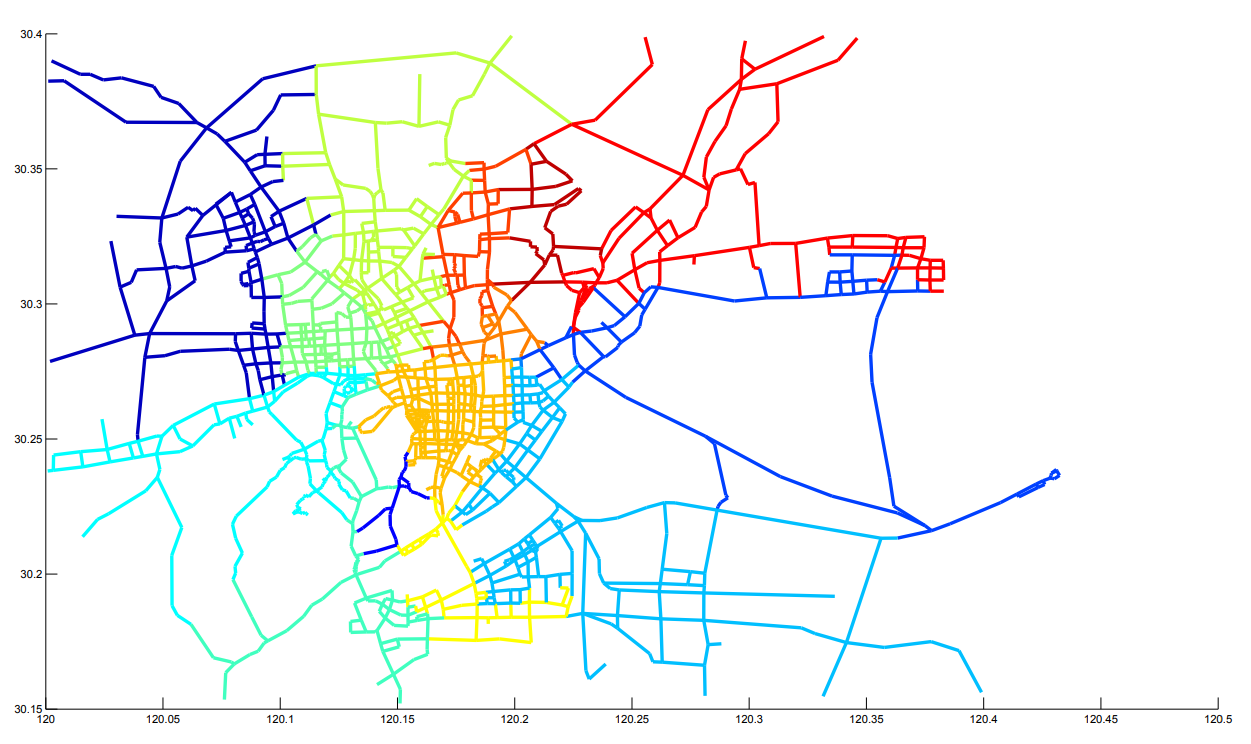}
  \caption{Decomposition of the road network.}
  \label{fig:cityDecomposition}
\end{figure}

We chose 10 taxis for each rank (expert, mid-level, and low) based on the
scores mentioned above and computed the expected features from the GPS
trajectories. We display the average expected features for the highest ranked
taxis in the left panel of \autoref{fig:city_values}. We used the computed
weight vector $w$ as a reward function and computed the value function for the
road network encoded as an MDP;  the resulting value function is displayed in
the right panel of \autoref{fig:city_values}. It is interesting to note that
our algorithm recognizes certain important passenger pickup areas (such as the
airport) even though the experts do not have very high expected features there.
The areas that have highest value are in accordance with places that have
intuitively higher taxi demand, such as bus stations and hospitals. Our method
is also able to recognize roads with poor passenger pickup records, such as in
the mountainous area, where people do not generally wait for taxis. Although it
seems somewhat surprising that the downtown area is given a relatively low
value, this is in accordance to the results found in \citep{liu10uncovering}, where they uncovered
the fact that the most successful taxi drivers do not always search for
passengers in areas with highest demand (such as in the downtown area), as
these areas are often prone to high levels of traffic. Using this value
function we can provide taxi drivers with passenger-finding strategies on a
much finer granularity than existing methods.

\begin{figure}
  \begin{subfigure}{0.5\textwidth}
    \centering
    \includegraphics[width=\textwidth]{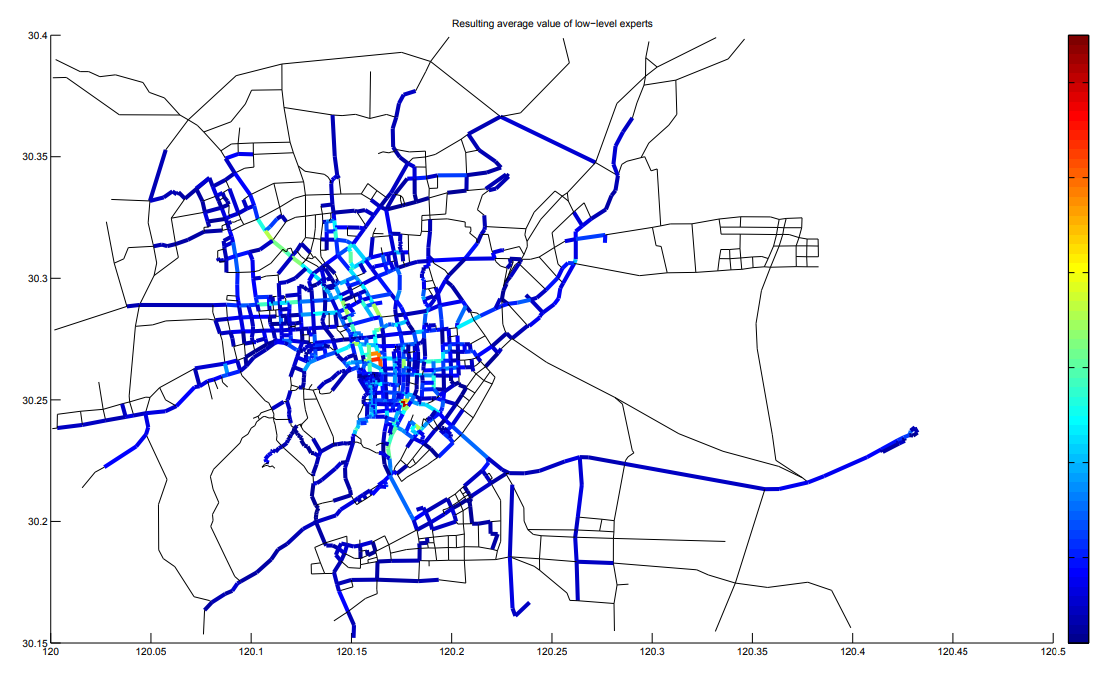}
  \end{subfigure}
  \begin{subfigure}{0.5\textwidth}
    \centering
    \includegraphics[width=\textwidth]{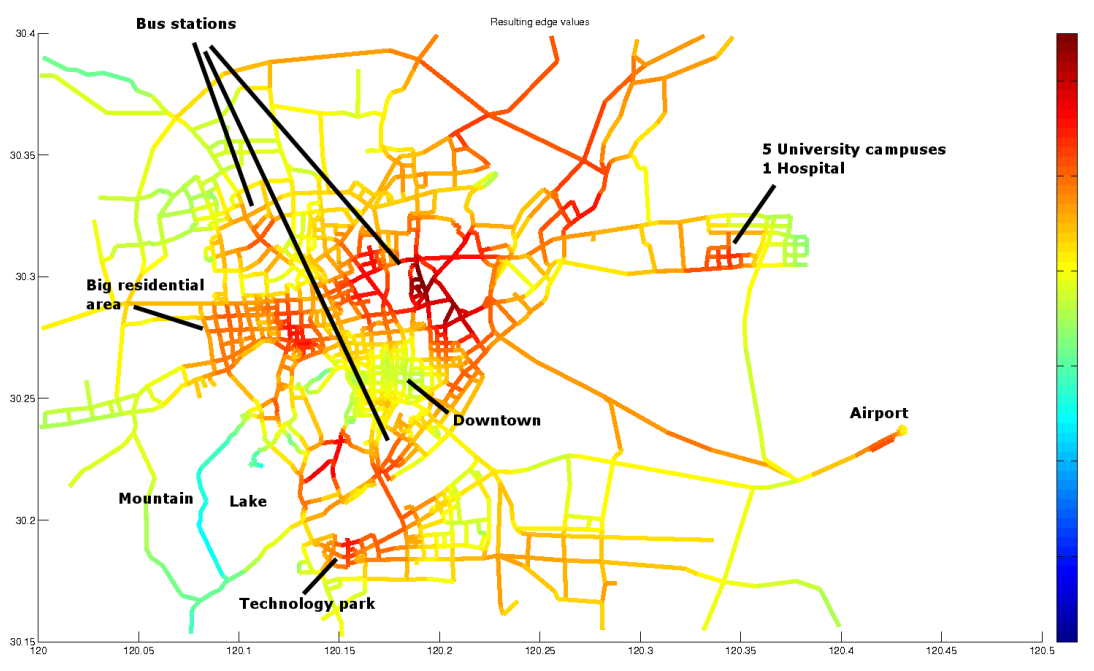}
  \end{subfigure}
  \caption{Average expected features for highest ranked taxis (left) and
  resulting value function using computed $w$ (right). Note that since each
  road segment has two possible orientations, we chose to display the one with
  highest value.}
  \label{fig:city_values}
\end{figure}

\section{Conclusion}
\label{sec:conclusion}
We have presented a novel method for apprenticeship learning, based on the
principle that agents should not only learn what to do, but should also learn
what {\em not} to do by observing non-experts’ behaviours. The illustration on
the 16x16 grid clearly demonstrates that learning solely from experts (as in
most apprenticeship algorithms) can lead to certain important aspects of the
true underlying reward function to remain ``hidden'' from the algorithm.
Furthermore, our approach is inherently unambiguous as it will always produce
the same weight vector, given a fixed set of expected features, thereby
overcoming solution degeneracy. This reduced variability in solutions is in
stark contrast to the method in \citep{abbeel04apprenticeship}, whose resulting solutions are greatly
dependent on the starting policy. An additional advantage of our approach over
closely related methods (such as \citep{abbeel04apprenticeship} and \citep{ratliff06maximum}) is that we are not required to
solve any MDPs at each iteration. Rather, our method produces a weight vector
$w$ based solely on the expected features of the demonstrators.

Although our results for passenger-finding strategies are very promising, we
are convinced we can do much better if we consider different time-slots in
isolation. In this paper we have examined the behaviour of the different
drivers throughout the weekdays, disregarding important differences between
different time slots (such as between rush hour and late night shifts).

We were fortunate to have a scoring function that produced a reliable ranking
of the taxi drivers for our problem. For other problems, however, it may be
difficult to properly rank the different demonstrators. The slack variables
$\varepsilon^r_i$ and $\varsigma^{r+1}_i$ provide a way to determine improperly
ranked demonstrators. One could take either an incremental or a pruning
approach. In the incremental approach one begins with a small set of
demonstrators and incrementally adds new ones (thereby providing more
information about the underlying reward function to the algorithm) as long as
the new demonstrator is not coupled with a positive $\varepsilon$ or
$\varsigma$ value. In the pruning approach, one could remove (or swap) the
expert with the highest $\varepsilon$ or $\varsigma$ value.

We have also used our algorithm to find optimal navigating routes through our
road network. Although there exist many algorithms for finding optimal routes
to a destination, these usually rely on a pre-existing reward function. Noting
that taxi drivers usually know the road network of a city better than regular
drivers, we would like to learn from their navigational behaviours. Finding a
good ranking function has proved somewhat daunting, as we often find certain
demonstrators improperly ranked. We have used the pruning approach mentioned
above to remove these outliers and this has greatly improved the output of our
algorithm, producing very promising preliminary results.  A difficulty with
this type of problem is that an expert’s expected features will probably not
cover the whole road network, resulting in many areas with poor navigational
suggestions. By learning from the behaviour of non-experts, our method is able
to cover a much larger area of the city.

\bibliography{irl}
\bibliographystyle{apalike}

\end{document}